\documentclass[twoside]{article}

\usepackage[accepted]{customstyle}
\usepackage[round]{natbib}

\usepackage{algorithm}
\usepackage{algpseudocode}
\usepackage{latexsym, amsmath, amscd, amssymb, amsthm}
\usepackage{bbm}
\usepackage{mathtools}
\usepackage{enumerate}
\usepackage{tabularx}

\algnewcommand{\LineComment}[1]{\State \(\triangleright\) #1}

\makeatletter
\newcommand{\multiline}[1]{%
  \begin{tabularx}{0.95\dimexpr\linewidth-\ALG@thistlm}[t]{@{}X@{}}
    #1
  \end{tabularx}
}
\makeatother

\DeclarePairedDelimiter\ceil{\lceil}{\rceil}

\newtheorem{lemma}{Lemma}
\newtheorem{theorem}[lemma]{Theorem}

\theoremstyle{definition}

\theoremstyle{remark}

\begin{document}

%

%

\twocolumn[

\aistatstitle{Risk-Aware Algorithms for Combinatorial Semi-Bandits} 

\aistatsauthor{Shaarad A. R. \And Ambedkar Dukkipati}

\aistatsaddress{} 
]

\begin{abstract}
In this paper, we study the stochastic combinatorial multi-armed bandit problem under semi-bandit feedback. While much work has been done on algorithms that optimize the expected reward for linear as well as some general reward functions, we study a variant of the problem, where the objective is to be risk-aware. More specifically, we consider the problem of maximizing the Conditional Value-at-Risk ($\mathsf{CVaR}$), a risk measure that takes into account only the worst-case rewards. We propose new algorithms that maximize the $\mathsf{CVaR}$ of the rewards obtained from the super arms of the combinatorial bandit for the two cases of Gaussian and bounded arm rewards. We further analyze these algorithms and provide regret bounds. We believe that our results provide the first theoretical insights into combinatorial semi-bandit problems in the risk-aware case.   \end{abstract}

\section{INTRODUCTION}


The multi-armed bandit framework provides a platform to study sequential decision making problems. Here, an agent has to adaptively choose among decisions available at each time, taking into account the feedback received till that time as well as a plan for future decisions. Such a framework can be used to formalize and solve problems from a wide variety of fields.
In a stochastic multi-armed bandit, choosing an arm gives the agent a reward sampled from a probability distribution corresponding to that arm. The objective of the agent is to maximize the expected rewards obtained over the entire time horizon. 



While  in the standard setting, at each time step the agent has to choose one arm (or take only one decision), many practical problems require multiple choices to be made at each time step. Such a situation arises in problems such as influence maximization~\citep{Chen:CombinatorialMABGeneralFramework}, cognitive radio networks \citep{Gai:LearningMultiUserChannelAllocations}, the stochastic version of problems such as the shortest path problem \citep{Gai:CombinatorialNetworkOptim} etc. In these problems, multiple arms can be chosen by the agent, possibly with some constraints on the combinations of arms that can be chosen. Such problems have been studied as the multi-armed bandit with multiple plays \citep{Anantharam:MultiplePlay} and the combinatorial bandit \citep{Gai:CombinatorialNetworkOptim}. In this paper, we consider the stochastic combinatorial bandit where each arm has a corresponding fixed reward distribution.

When a set of arms is chosen at some time instant, if only the total or combined reward from the chosen set of arms is revealed to the agent, the setting is called bandit feedback. In this paper, we study the problem under semi-bandit feedback, in which the rewards from all the chosen arms are revealed to the agent.

In the standard bandit formulation, the objective of the agent is to maximize the expected sum of all the rewards obtained by the agent over the entire time horizon. In this setting, each arm is judged solely based on the expectation of its reward distribution. However, there may exist situations in which we wish to consider the entire reward distribution instead of just its expectation. For example, we might wish to make decisions that give a consistent reward, or a good worst-case reward. This lead to the development of various risk-awake strategies or algorithms. 

While there are many risk measures that have been proposed, in this paper, we consider a popular and commonly used risk measure called Conditional Value-at-Risk ($\mathsf{CVaR}$). This is  parameterized by some probability level $\alpha \in (0,1)$. Intuitively, the $\mathsf{CVaR}$ of a probability distribution is the expected value of a random variable sampled from the distribution, conditioned on the sample being atmost the $\alpha-$quantile of the distribution. Thus, the decision with the best $\mathsf{CVaR}$ is the decision with the best expected reward in the worst case, where the ``worst case" is quantified by the parameter $\alpha$.

We propose and analyze algorithms to tackle a stochastic multi-armed bandit with semi-bandit feedback, where the agent suffers regret whenever a subset of arms (called a super arm) is chosen that does not have the maximum possible $\mathsf{CVaR}$ among all allowed super arms. The $\mathsf{CVaR}$ considered here is that of the reward distributions of the super arms, where the reward from a super arm is the sum of the rewards from each of its constituent arms. We study this problem for the two cases in which the individual arm rewards are (i) normally distributed, and (ii) bounded.

\section{RELATED WORK}
The multi-armed bandit problem with multiple plays has been first studied by \cite{Anantharam:MultiplePlay}, who provided an algorithm and a lower bound on the regret for the setting in which the reward distributions of the arms are members of a parameterized family of probability distributions. An optimal Thompson sampling algorithm for solving this problem has been proposed \citep{Koniyama:OptimalRegretAnalysisThompsonMABMultiplePlays} for the case of Bernoulli arm rewards.

\cite{Gai:LearningMultiUserChannelAllocations} formulated the combinatorial bandit problem and proposed an algorithm for the problem of optimally allocating channels to users in a cognitive radio network. This work was generalized to the LLR (Learning with Linear Rewards) algorithm \citep{Gai:CombinatorialNetworkOptim} for combinatorial bandits in which each super arm is characterized by a set of coefficients corresponding to the arms of the bandit, and the reward from a super arm is the linear combination of the rewards from the individual arms.

This has been extended \citep{Chen:CombinatorialMABGeneralFramework} to more general non-linear reward functions that are monotonic and smooth. However, their algorithm CUCB (Combinatorial Upper Confidence Bound) only applies to those reward functions that satisfy the property that the expected value of the reward obtained from choosing a super arm depends only on the expected values of the reward distributions of the individual arms. In other words, the expected reward from a super arm should be fully characterized by the means of the individual arm reward distributions. Because of this assumption, the problem can be solved by constructing upper confidence bounds for each of the arm reward means and using these values to determine the optimal super arm.

The above assumption is removed by \cite{Chen:CombMABGeneralRewardFuncs} who proposed SDCB (Stochastically Dominant Confidence Bound) that can also deal with bounded monotone reward functions that depend on the entire probability distributions of the arm rewards and not just the means. SDCB constructs a new probability distribution for each arm that acts as a stochastically dominant confidence bound for the actual unknown arm reward distribution, and uses these stochastically dominant distributions to determine the optimal super arm. We adapt this technique to tackle the case of general bounded rewards.

\cite{Kveton:TightRegretForCMAB} provided tight regret bounds for the case of bounded linear rewards, which are further improved by ESCB (Efficient Sampling for Combinatorial Bandits) \citep{Combes:CombBanditsRevisited} for the special case of Bernoulli rewards. For general bounded rewards, \cite{Merlis:BatchSizeIndepRegretForCMAB} proposed an algorithm BC-UCB (Bernstein Combinatorial Upper Confidence Bound) that constructs confidence bounds for the arms that also use the empirical variance of the observed samples. Their regret bound is independent of the size of the super arms, and instead depends on the smoothness parameters of the non-linear reward function.

\cite{Merlis:TightLowerBoundForCMABs} provided tight lower bounds (up to logarithmic factors) for the setting in which the arm rewards are bounded and the super arm reward is a monotone smooth function of the expected arm rewards, and also for linear rewards.

The combinatorial bandit problem has also been tackled using a Thompson sampling approach by \cite{WangChen:2020:ThompsonSamplingForCombinatorialSemiBandits}, whose algorithm achieves the same theoretical regret as the UCB-based algorithm CUCB \citep{Chen:CombinatorialMABGeneralFramework}, and matches the regret lower bound \citep{Kveton:TightRegretForCMAB} for linear reward functions.

Risk-awareness has first been studied for multi-armed bandits under the mean-variance criterion \citep{Sani:RiskAversionInMABs} and the $\mathsf{MIN}$ and $\mathsf{CVaR}$ criteria \citep{Galichet:ExploVsSafetyRiskAwareMABs} for bounded rewards. An algorithm with provable guarantees on the $\mathsf{CVaR}$ regret was proposed for bounded rewards by \cite{Galichet:ContribsToMABsRiskAwareness}.

\cite{BhatPrashanth:ConcentrationOfRiskMeasuresWassersteinDistanceApproach} provide concentration results for $\mathsf{CVaR}$ estimation, and thereby a UCB-like algorithm for optimizing $\mathsf{CVaR}$ regret, for the case of subgaussian arm rewards. An entirely distribution oblivious algorithm that works with any arm reward distribution has been proposed and analyzed by \cite{Jagannathan:DistributionObliviousRiskAwareAlgsForMABsWithUnboundedRewards}. Their algorithms require very mild assumptions on the rewards and work for unbounded rewards and rewards with heavy tails.

An algorithm, U-UCB, that can deal with general risk metrics including $\mathsf{CVaR}$ has been proposed by \cite{Cassel:AGeneralApproachToMABsUnderRiskCriteria}. They have also derived performance guarantees for their algorithm \citep{Cassel:AGeneralFrameworkForBanditProblemsBeyondCumulativeObjectives} with respect to a stronger form of regret than in other works that deal with $\mathsf{CVaR}$. 
To the best of our knowledge, our work is the first work that develops and studies risk aware algorithms in a combinatorial setting. 

\section{PROBLEM FORMULATION}
Consider a bandit with $K$ arms and time horizon $T$. Let $[K] = \{1, 2, \dots, K\}$ be the set of all arms and $2^{[K]}$ be its power set. Each element $a \in 2^{[K]}$ is a set of arms and is called a super arm. At each time step, the agent has to pick a super arm from a given fixed set of super arms $\mathcal{A} \subseteq 2^{[K]}$. In other words, $\mathcal{A}$ defines the constraints on the possible combinations of arms that can be chosen by the agent. Let $L$ be the maximum number of arms that can be simultaneously selected, i.e, $L = {\displaystyle \max_{a \in \mathcal{A}}} |a|$.

Each arm $i \in [K]$ has an associated probability distribution $\mathcal{D}_i$ for its reward. When a super arm $a \in \mathcal{A}$ is selected, a reward is sampled from each arm $i \in a$ and the agent receives the sum of these sampled rewards.

Note on notation: We use the symbol $\oplus$ and $\bigoplus$ to denote the binary and $n$-ary sum of probability distributions, and so denote the distribution of the reward of super arm $a$ as $\mathcal{D}_a$ or $\mathcal{D}_{\bigoplus_{i \in a} \mathcal{D}_i}$. Depending on the context, for any distribution $\mathcal{D}$ with some subscript/superscript, we denote its corresponding cumulative distribution function and probability density/mass function as $F$ and $f$ respectively with the same subscript/superscript.

For any random variable $X \sim D$ with probability distribution $\mathcal{D}$ and cumulative distribution function $F_X$, the Conditional Value at Risk, for risk level $\alpha \in (0,1)$, is defined as
$$
\mathsf{CVaR}_\alpha(X) = \frac{1}{\alpha} \int_0^\alpha \mathsf{VaR}_p(X) dp,
$$
where $\mathsf{VaR}_p$ is the Value-at-Risk at risk level $p$. Value-at-Risk is defined as
$$
\mathsf{VaR}_\alpha(X) = x_\alpha = \inf \{ x \in \text{Support}(\mathcal{D}) : F_X(x) \geq \alpha  \}.
$$
Intuitively, for continuous probability distributions, $\mathsf{VaR}_\alpha$ is just the $\alpha$ quantile of distribution, and $\mathsf{CVaR}_\alpha$ is the expected value of the random variable conditioned on its value being at most $\mathsf{VaR}_\alpha$. It is the same for discrete distributions, except for an additional term that takes into account and subtracts the extra probability mass at the $\mathsf{VaR}$ that might go beyond the cutoff $\alpha$.

The goal of the agent is to choose the super arm with the best Conditional Value-at-Risk at every time step. Whenever a super arm with a suboptimal $\mathsf{CVaR}$ is chosen, the agent suffers a regret. So, the goal is to minimize this $\mathsf{CVaR}$ regret, defined by
$$
{\mathcal{R}}_{\mathsf{CVaR}_\alpha}(T) = \mathbb{E} \left[ \sum_{t=1}^T \left( \mathsf{CVaR}_\alpha(a^*) - \mathsf{CVaR}_\alpha(a_t) \right) \right],
$$
where $\mathsf{CVaR}_\alpha(a)$ is the $\mathsf{CVaR}$ of the reward distribution $\mathcal{D}_a$ of super arm $a$, $a_t$ is the super arm chosen by the agent at time $t$, and $a^*$ is the super arm with the highest $\mathsf{CVaR}$. The above regret can also be written as
$$
{\mathcal{R}}_{\mathsf{CVaR}_\alpha}(T) = \sum_{a \in \mathcal{A}} \mathbb{E} \left[ T_a(T) \right] \Delta_a,
$$
where $T_a(T)$ is the number of times super arm $a$ has been chosen till time $T$, and
$$
\Delta_a = \mathsf{CVaR}_\alpha(a^*) - \mathsf{CVaR}_\alpha(a)
$$
is the $\mathsf{CVaR}$ gap of super arm $a$. Further, let $\Delta_{min}$ and $\Delta_{max}$ denote the minimum and maximum nonzero $\Delta_a$ among all arms $a \in \mathcal{A}$.

\section{ALGORITHMS}
The naive way to solve a combinatorial bandit problem is to treat each super arm as an arm, thereby reducing the problem to a standard multi-armed bandit problem. 
The disadvantage of this approach is that the regret upper bound depends linearly on the number of super arms, which itself may depend exponentially on the number of arms. This is because such a naive approach ignores the dependence between the super arms. It has been shown that an algorithm and regret analysis that takes into account this dependence can yield a regret upper bound that is only polynomial in the number of arms~\citep{Gai:CombinatorialNetworkOptim}.

Thus, the crux of such an algorithm is the way in which samples from individual arms are used to construct estimates of the relevant properties of the super arm reward distributions. This depends on the nature of the arm reward distributions. In this paper, we propose risk-aware algorithm in two cases: (i) normally distributed, and (ii) bounded arm rewards.

\subsection{Gaussian rewards}
Assume that each arm $i \in [K]$ has Gaussian reward distribution with mean $\mu_i$ and variance $\sigma_i^2$, and that the rewards of all arms are independent. This implies a super arm $a \in \mathcal{A}$ has a Gaussian reward distribution with mean $\mu_a = \sum_{i \in a} \mu_i$ and variance $\sigma_a^2 = \sum_{i \in a} \sigma_i^2$. The $\mathsf{CVaR}$ of this super arm is then given by
$$
\mathsf{CVaR}_\alpha(\mathcal{D}_a) = \mu_a - \frac{\sigma_a}{\alpha} \varphi \left( \Phi^{-1}(\alpha) \right),
$$
where $\varphi$ and $\Phi$ are the probability density function and cumulative distribution function of the standard normal distribution. Thus, to estimate the $\mathsf{CVaR}$ of a super arm, it is sufficient to estimate the mean and standard deviation of the super arm.

Since both the mean and the variance are unknown, for obtaining suitable finite-sample guarantees on our estimates, we impose an additional assumption that the agent has access to prior information on an upper and lower bound on the variance of any arm. We assume that there exist known constants $M, N > 0$ such that $N^2 < \sigma_i^2 < M^2$ for every arm $i \in [K]$.

This allows us to construct an upper confidence bound on the mean of each arm, and a lower confidence bound on the variance of each arm, in turn allowing us to use these to create upper confidence bounds on the $\mathsf{CVaR}$ of all super arms. During each time step, we select the super arm with the highest $\mathsf{CVaR}$ upper confidence bound, as detailed in Algorithm \ref{alg:cucb_gaussian}. An upper bound for the $\mathsf{CVaR}$ regret suffered by this algorithm is given in the following theorem.

\begin{algorithm}
\begin{algorithmic}
\Require $M$, $N$
\State Pick each super arm so that all the arms have atleast two rewards obtained
\For{each time $t$ till $T$}
	\State\multiline{Estimate sample mean $\hat{\mu}_i$ and sample variance $\hat{s}^2_i$ for each arm $i$}
	\LineComment{\multiline{Construct corresponding confidence bounds
	$\tilde{\mu}_i$ and $\tilde{s}_i^2$}}
	\State\multiline{$m_i(t-1) \leftarrow$ Number of rewards obtained from arm $i$ till time $t-1$}
    \State $\tilde{\mu}_i \leftarrow \hat{\mu}_i + 2M \sqrt{\frac{(L+1) \log (t-1)}{m_i(t-1)}}$
    \State $D_{i,t-1} \leftarrow M^2 \sqrt{\frac{2(L+1)\log (t-1)}{m_i(t-1) - 1} + \frac{4(L+1)^2 (\log (t-1))^2}{(m_i(t-1) - 1)^2}}$
    \State $\tilde{s}^2_i \leftarrow \max \left\{ \left( \hat{s}_{i, m_i(t-1)}^2 - D_{i, t-1} \right), N^2 \right\}$
	\LineComment{\multiline{Construct upper confidence bound for $\mathsf{CVaR}$ for each super arm $a$, as}}
	\State $\sum_{i \in a} \tilde{\mu}_i - \frac{\sqrt{\sum_{i \in a} \tilde{s}^2_i }}{\alpha} \varphi(\Phi^{-1}(\alpha))$
	\State Pick the super arm with best UCB.
\EndFor
\end{algorithmic}
\caption{Algorithm $\mathsf{CVaR}_\alpha$-CUCB-G}
\label{alg:cucb_gaussian}
\end{algorithm}

\begin{theorem}
\begin{align*}
{\mathcal{R}}&_{\mathsf{CVaR}_\alpha}(T) \lesssim \left( 2K + \frac{2}{3} \pi^2 LK \right) \Delta_{max} \\
& + \frac{4M^2 \sqrt{L}(L+1) K \log T \Delta_{max}}{\Delta_{min}} \max\bigg\{ \frac{16 L\sqrt{L}}{\Delta_{min}},\\
& \qquad \qquad \qquad \qquad \qquad \qquad \qquad \frac{3 \varphi (\Phi^{-1}(\alpha))}{\alpha N} \bigg\}.
\end{align*}
\label{thm:gaussian}
\end{theorem}

\subsection{Bounded rewards}

In this setting, we assume that the rewards from each of the arms are non-negative and bounded above by a known upper bound. Without any loss in generality, we can assume that the rewards of each arm fall in the interval $[0,1]$.

In the previous subsection, the Gaussian assumption for the rewards of the arms significantly simplifies the estimation of $\mathsf{CVaR}$ of the super arms since the $\mathsf{CVaR}$ of each super arm reduces to a simple function of the parameters of a parametrized probability distribution that can be easily estimated, i.e, the mean and standard deviation, which in turn can be easily constructed from the mean and standard deviation of the individual arms that constitute that super arm. But for general nonparametric distributions, obtaining confidence intervals for the $\mathsf{CVaR}$ of a super arm is much less straightforward since it cannot be calculated as a simple function of the $\mathsf{CVaR}$'s of the constituent arms.

For estimating the $\mathsf{CVaR}$ of a super arm, it is sufficient to construct a probability distribution that is sufficiently close to the actual underlying probability distribution of the super arm rewards, and calculate the $\mathsf{CVaR}$ of this "approximate" distribution. However, for constructing a multi-armed bandit algorithm that minimized the $\mathsf{CVaR}$ regret, we need to construct an upper confidence bound on the $\mathsf{CVaR}$ of the super arms. This can be done by constructing a probability distribution that is close but stochastically dominates the super arm reward distribution. This stochastically dominant probability distribution, in turn, can be constructed by constructing stochastically dominant distributions for each of the individual arms, and for each super arm, calculating the convolution of those distributions that correspond to its constituent arms.

Thus, if $\hat{F}_i$ is the cumulative distribution function of the empirical distribution formed by all the samples obtained from arm $i$, we construct a corresponding stochastically dominant distribution $\tilde{F}_i$ by subtracting a constant throughout the domain of $F$ below $1$, the known upper bound for the actual distribution being estimated. This can be used to construct a stochastically dominant distribution for each super arm $a$ as $\tilde{F}_a = \bigoplus_{i \in a} \tilde{F}_i$. This procedure is listed in Algorithm \ref{alg:bounded}, and an upper bound for the regret of this algorithm is given by the following theorem.

\begin{algorithm}
\begin{algorithmic}
\State Pick each super arm so that all the arms have atleast one reward obtained
\For{each time $t$ till $T$}
	\For{each arm $i \in [k]$}
	\State \multiline{$\hat{F}_{i,t-1} \leftarrow$ Empirical distribution of the rewards obtained till time $t-1$}
	\State \multiline{$T_{i,t-1} \leftarrow$ Number of rewards obtained from arm $i$ till time $t-1$}
	\State $C_{i,t-1,m_i(t-1)} \leftarrow \sqrt{\frac{3 \log (t)}{2T_{i, t-1}}}$
	\State $\tilde{F}_i(x) \leftarrow \begin{cases} ( \hat{F}_i(x) - C_{i, t-1, m_i(t-1)} )^+,x < 1 \\ 1, \qquad \qquad \qquad \qquad \text{ otherwise} \end{cases}$
	\State \multiline{Calculate empirical CDF for each super arm $a$ as $\tilde{F}_a \leftarrow F_{ \bigoplus_{i \in a} \tilde{F}_i}$}
	\EndFor
	\State Calculate $\mathsf{CVaR}_\alpha(\tilde{F}_a)$ for each super arm $a$
	\State Pick super arm with best $\mathsf{CVaR}$
\EndFor
\end{algorithmic}
\caption{Algorithm $\mathsf{CVaR}_\alpha$-SDCB}
\label{alg:bounded}
\end{algorithm}

\begin{theorem}
The regret for algorithm \textsc{$\mathsf{CVaR}_\alpha$-SDCB} satisfies
\begin{align*}
{\mathcal{R}}_{\mathsf{CVaR}_\alpha}(T) \leq C \frac{L^3}{\alpha^4} \log T \sum_{i \in a_B} \frac{1}{\Delta_{i,min}} + \left( 1 + \frac{\pi^2}{3} \right) K \Delta_{max},
\end{align*}
where $a_B$ is the set of arms contained in at least one suboptimal super arm, and $\Delta_{i, min} = \min \{ \Delta_a (\neq 0) : i \in a \}$.
\label{thm:regret_bounded}
\end{theorem}

The above algorithm requires computing a stochastically dominant probability distribution for each super arm in $\mathcal{A}$, which involves computing the probability distribution of a sum of atmost $L$ discrete probability distributions. Let $\mathcal{D}_1, \dots, \mathcal{D}_l$ be $l (\leq L)$ distributions whose sum has to be computed for some super arm, and let Supp$(\mathcal{D}_i)$ be the support of $\mathcal{D}_i$.

Calculating the distribution of $\mathcal{D}_1 \oplus \mathcal{D}_2$ requires performing a convolution of $f_1$ and $f_2$ and hence involves atmost $|\text{Supp}(\mathcal{D}_1)| |\text{Supp}(\mathcal{D}_2)|$ computations. Further, $|\text{Supp}(\mathcal{D}_1 \oplus \mathcal{D}_2)| \leq |\text{Supp}(\mathcal{D}_1)| |\text{Supp}(\mathcal{D}_2)|$ with equality occuring in the worst case. This means that computing $(\mathcal{D}_1 \oplus \mathcal{D}_2) \oplus \mathcal{D}_3$ requires $|\text{Supp}(\mathcal{D}_1 \oplus \mathcal{D}_2)| |\text{Supp}(\mathcal{D}_3)| \leq |\text{Supp}(\mathcal{D}_1)| |\text{Supp}(\mathcal{D}_2)| |\text{Supp}(\mathcal{D}_3)|$ computations, with the support of $\mathcal{D}_1 \oplus \mathcal{D}_2 \oplus \mathcal{D}_3$ satisfying $|\text{Supp}((\mathcal{D}_1 \oplus \mathcal{D}_2) \oplus \mathcal{D}_3)| \leq |\text{Supp}(\mathcal{D}_1)| |\text{Supp}(\mathcal{D}_2)| |\text{Supp}(\mathcal{D}_3)|$, again with equality occuring in the worst case. 

Reasoning this way, it is clear that the number of computations required for just the final step of calculating the stochastically dominant distribution for each super arm $a \in \mathcal{A}$ at time $t$ is
$$
\prod_{i \in a} |\text{Supp}(\tilde{F}_{i, t-1})| \leq \left( \max_{i \in a} |\text{Supp}(\tilde{F}_{i, t-1})| \right)^L.
$$
This quantity is exponential in $L (\leq k)$ and might cause the algorithm to become computationally expensive for large $L$, or for a large time horizon that may cause $\max_{i \in a} |\text{Supp}(\tilde{F}_{i, t-1})|$ to become large for continuous probability distributions. To mitigate this, we propose a discretized algorithm.

\subsection{Discretized algorithm}

The problem of high computational complexity occurs because the support of sums of discrete probability distributions keeps expanding with the number of distributions in the sum. This can be solved by discretizing the distributions further and allowing the random variables involved to take only certain values, thereby limiting the support.

More specifically, we choose some small real number $\epsilon > 0$, and at each time $t$, we ``round up" each distribution $\tilde{F}_{i, t-1}$ to a new distribution $F'_{i,t-1}$ by moving the probability mass at each point $x \in \text{Supp}(\tilde{F}_{i,t-1})$ to the smallest point $x' \geq x$ that is a integral multiple of $\epsilon$, i.e, $x' = \ceil{\frac{x}{\epsilon}} \epsilon$. When probability mass from multiple points is moved to the same multiple of $\epsilon$, the individual mass values are added up to obtain the total probability mass at the final point.

For our problem of minimizing the $\mathsf{CVaR}$ regret, we choose parameter $\epsilon = \frac{\alpha}{(L+1)T}$, which requires the knowledge of $T$, unlike the previous algorithms. The resultant algorithm \textsc{D-CVaR$_\alpha$-SDCB} is detailed in Algorithm \ref{alg:bounded_discrete}. An upper bound for the regret of this algorithm is given by the following theorem:

\begin{theorem}
The algorithm \textsc{D-CVaR}$_\alpha$\textsc{-SDCB} has $\mathsf{CVaR}$ regret that satisfies
\begin{align*}
{\mathcal{R}}_{\mathsf{CVaR}_\alpha} (T) \leq 2 & + \left( 1 + \frac{\pi^2}{3} \right) K\Delta_{max} \\
& + C \frac{L^3}{\alpha^4} \log T \sum_{i \in a_B} \frac{1}{\Delta_{i,min}}.
\end{align*}
\label{thm:regret_bounded_approximate}
\end{theorem}

\begin{algorithm}
\begin{algorithmic}
\Require $T$
\State Pick each super arm so that all the arms have at least one reward obtained
\For{each time $t$ till $T$}
	\For{each arm $i \in [k]$}
	\State \multiline{$\hat{F}_{i,t-1} \leftarrow$ Empirical distribution of the rewards obtained till time $t-1$}
	\State \multiline{$T_{i,t-1} \leftarrow$ Number of rewards obtained from arm $i$ till time $t-1$}
	\State $C_{i,t-1,m_i(t-1)} \leftarrow \sqrt{\frac{3 \log (t)}{2T_{i, t-1}}}$
    \State $\tilde{F}_i(x) \leftarrow \begin{cases} ( \hat{F}_i(x) - C_{i, t-1, m_i(t-1)} )^+,x < 1 \\ 1, \qquad \qquad \qquad \qquad \text{ otherwise} \end{cases}$
	\State Discretize $\tilde{F}_i$ to $F'_i$
	\EndFor
	\State \multiline{Calculate empirical CDF for each super arm $a$ as $F'_a \leftarrow F_{\bigoplus_{i \in a} F'_i }$}
	\State Calculate $\mathsf{CVaR}_\alpha(F'_a)$ for each super arm $a$
	\State Pick super arm with best $\mathsf{CVaR}$
\EndFor
\end{algorithmic}
\caption{Algorithm \textsc{D-CVaR}$_\alpha$\textsc{-SDCB}}
\label{alg:bounded_discrete}
\end{algorithm}

\section{DISCUSSION}

Analysis of the regret of $\mathsf{CVaR}_\alpha$-CUCB-G requires two-sided confidence bounds on both the mean and standard deviation of the normal distribution associated with each arm of the bandit, which is unusual for standard multi-armed bandit problems. When dealing with Gaussian random variables, \cite{AuerBianchiFischer:2002:FinitetimeAnalysisOfTheMultiarmedBanditProblem} used a concentration inequality for $\chi^2$ random variables which is a conjecture that they verified numerically. For constructing confidence intervals as part of our analysis, we use the concentration inequalities that are a corollary of Lemma 1 in \cite{Laurent:AdaptiveEstimOfQuadraticFunctionalByModelSelection}. However, the size of such a confidence interval for the variance of the Gaussian random variable itself depends on the true unknown variance. For dealing with this, we assume knowledge of upper and lower bounds on the variance of each arm reward distribution.

A discretization approach similar to ours in Algorithm \ref{alg:bounded_discrete} was used by \cite{Chen:CombMABGeneralRewardFuncs} to decrease the worst case space and time complexity, from $\Theta(T)$ and $\Theta(T^2)$ to $\Theta(\sqrt{T})$ and $\Theta(T^{3/2})$ respectively, of the memory usage and computations related to maintaining $\hat{F}_i$ and $\tilde{F}_i$ for each arm $i$. However, in our paper, we use the discretization mainly for tackling the computation of the probability distributions corresponding to the super arms.



\section{CONCLUSION}

In this paper, we studied risk-awareness for the problem of stochastic combinatorial multi-armed bandits under semi-bandit feedback. Specifically, we proposed algorithms for optimizing the Conditional Value-at-Risk of the super arms of the combinatorial bandit for the cases of Gaussian and bounded arm rewards. We analyzed the regret of these algorithms to show their theoretical superiority over a naive approach that does not take into account the combinatorial structure of the problem.

\bibliography{bibfile}

\clearpage
\onecolumn
\appendix

\section{Proof of Theorem \ref{thm:gaussian}}

\subsection{Confidence intervals}

At each time step, the agent has some samples from the Gaussian distribution corresponding to each arm of the bandit. These samples are used to construct the sample mean and sample variances of the arms, which in turn are used to construct the mean and variances, and thereby $\mathsf{CVaR}$, of the super arms.

For such an algorithm and its analysis, we need upper and lower confidence bounds on the sample means and sample variances of every arm, and corresponding concentration inequalities.

Let $x_1, \dots, x_n$ be $n$ samples obtained from arm $i$ till time $t$. Since we are considering only a specific arm here, we can ignore the arm subscript for this part of the discussion. Let the distribution of each of these samples be $\mathcal{N}(\mu, \sigma^2)$. Let $\hat{\mu}_{n}$ and $\hat{s}_{n}^2$ be the sample mean and sample variance.

The concentration of the sample mean is obtained using the Hoeffding's inequality as
\begin{align*}
\mathbb{P} \left( | \hat{\mu}_{n} - \mu | \geq 2 \sigma \sqrt{\frac{(L+1) \log t}{n}} \right) \leq 2 \exp \left( - \frac{n}{2 \sigma^2} 4 \sigma^2 \frac{(L+1) \log t}{n} \right) = 2t^{-2(L+1)}.
\end{align*}
However, since $\sigma$ is unknown, we use its upper bound $M$, obtaining a confidence interval for arm $i$ of radius $C_{i,t,n} = 2 M \sqrt{\frac{(L+1) \log t}{n}}$.

For the sample variance, we know that $\hat{s}_n^2$ can be written as $\hat{s}_n^2 = \frac{\sigma^2}{n-1} X$ for some $\chi_{n-1}^2$ random variable $X$. We have the following concentration inequalities for $\chi_k^2$ random variables \citep{Laurent:AdaptiveEstimOfQuadraticFunctionalByModelSelection}
\begin{align*}
\mathbb{P} \left( X - k \geq 2 \sqrt{kx} + 2x \right) & \leq \exp(-x), \\
\mathbb{P} \left( k - X \geq 2 \sqrt{kx} \right) & \leq \exp(-x).
\end{align*}
For the lower confidence bound on the variance,
\begin{align*}
\mathbb{P} \left( \hat{s}_n^2 > \sigma^2 + \epsilon \right) = \mathbb{P} \left( X - (n-1) > \frac{\epsilon (n-1)}{\sigma^2} \right).
\end{align*}
For applying the first concentration inequality from \cite{Laurent:AdaptiveEstimOfQuadraticFunctionalByModelSelection}, we have to choose a suitable $x$ that satisfies $\frac{(n-1) \epsilon}{\sigma^2} \geq 2 \sqrt{(n-1)x} + 2x$. This gives rise to a quadratic inequality in $x$ and can be solved in terms of $\epsilon$ to give
$$
\mathbb{P} \left( \hat{s}_n^2 > \sigma^2 + \epsilon \right) < \exp \left[ - \frac{(n-1)}{2} \left( \sqrt{1 + \frac{4 \epsilon^2}{\sigma^4}} - 1 \right) \right].
$$
Similarly, the upper confidence inequality can be seen to be
$$
\mathbb{P} \left( \hat{s}_n^2 < \sigma^2 - \epsilon \right) < \exp \left( - \frac{(n-1) \epsilon^2}{4 \sigma^4} \right).
$$
Now, for the purpose of the regret analysis, we want these probabilities on the right hand side to be less than $t^{-2(L+1)}$. This requires suitable choices of the value of $\epsilon$, which turn out to be $\sigma^2 \sqrt{\frac{2(L+1)}{(n-1)} \log t + \frac{4(L+1)^2}{(n-1)^2} (\log t)^2}$ in the first case and $2 \sigma^2 \sqrt{\frac{2(L+1) \log t}{n-1}}$ in the second case. Finally, the $\sigma^2$ terms in these expressions are replaced by $M^2$ since the value of the variance is unknown (and is, in fact, the unknown quantity being estimated in the first place). This leads to a confidence interval $(\sigma^2_{l,n}, \sigma^2_{u,n})$, where
\begin{align*}
\sigma^2_{l, n} & = \hat{s}_n^2 - M^2 \sqrt{\frac{2(L+1)}{(n-1)} \log t + \frac{4(L+1)^2}{(n-1)^2} (\log t)^2}, \\
\sigma_{u,n}^2 & = \hat{s}_n^2 + M^2 \sqrt{\frac{2(L+1) \log t}{n-1}}
\end{align*}
\subsection{Regret}

The agent incurs regret whenever a suboptimal super arm is picked. However, directly counting the number of pulls of each suboptimal super arm and adding them up does not give a good picture of the cumulative regret due to the combinatorial nature of the bandit. Instead, we decompose the regret in terms of the underlying arms of the bandit, as in \cite{Gai:CombinatorialNetworkOptim}.

Whenever a suboptimal super arm is pulled, we count that as a suboptimal pull for the arm in the super arm with the least number of pulls at that time step. The total regret is then calculated as the sum of the regrets due to each arm. The regret due to each arm depends on the number of suboptimal pulls of that arm counted in the described manner, given by
\begin{align*}
    T_i(T) & \leq 2 + l + \sum_{t=l+1}^T \mathbbm{1} \{ i\text{'th arm was counted as part of a suboptimal super arm selection } a_t \text{ at time }t,\\
	& \quad \qquad \qquad \qquad \qquad \qquad m_i(t-1) > l \},
\end{align*}
where the first term $2$ corresponds to the first phase of the algorithm when all the arms are explored at least twice, and $m_i(t-1)$ is the number of times arm $i$ has been pulled through some super arm till time $t-1$.
\begin{align*}
T_i(T) & \leq 2 + l + \sum_{t=l+1}^T \mathbbm{1} \{ i\text{'th arm was counted as part of a suboptimal super arm selection } a_t \text{ at time }t,\\
	& \quad \qquad \qquad \qquad \qquad \qquad m_i(t-1) > l \} \\
	& = 2 + l + \sum_{t=l}^{T-1} \mathbbm{1} \{ i\text{'th arm was counted as part of a suboptimal super arm selection } a_{t+1} \text{ at time }t+1,\\
	& \quad \qquad \qquad \qquad \qquad \qquad m_i(t) > l \} \\
	& \leq 2 + l + \sum_{t=l}^{T-1} \Bigg[ \mathbbm{1} \left\{ \sum_{i \in a^*} \hat{\mu}_i + C_{i,t,m_i(t)} - \frac{\hat{\sigma}_l^{(a^*)}}{\alpha} \varphi(\Phi^{-1}(\alpha)) < \sum_{i \in a^*} \mu_i - \frac{\sigma^{(a^*)}}{\alpha} \varphi(\Phi^{-1}(\alpha)) \right\} +\\
	&  \qquad \qquad \qquad \mathbbm{1} \left\{ \sum_{i \in a_t} \mu_i - \frac{\sigma^{(a_t)}}{\alpha} \varphi(\Phi^{-1}(\alpha)) < \sum_{i \in a_t} \hat{\mu}_i - C_{i,t,m_i(t)} - \frac{\sigma_u^{(a_t)}}{\alpha} \varphi(\Phi^{-1}(\alpha)) \right\} \\
	& \qquad \qquad + \mathbbm{1} \left\{ \mathsf{CVaR}_\alpha(a^*) - \mathsf{CVaR}_\alpha(a_t) < 2 \sum_{i \in a_t} C_{i,t,m_i(t)} + \frac{\sigma_u^{(a_t)} - \sigma_l^{(a_t)}}{\alpha} \varphi(\Phi^{-1}(\alpha)), m_i(t) > l \right\} \Bigg]
\end{align*}
The first two terms are bounded by $2 t^{-2(L+1)}$ by the way the confidence intervals i.e $C_{i,t,m_i(t)}$ and $\sigma_u, \sigma_l$ are defined. For the second term, either \\ $\frac{1}{2} \Delta_{min} < 2 \sum_{i \in a_t} C_{i,t,m_i(t)}$ or $\frac{1}{2} \Delta_{min} < \frac{\sigma_u^{(a_t)} - \sigma_l^{(a_t)}}{\alpha} \varphi(\Phi^{-1}(\alpha))$.\\

Now, if $m_i(t) > \frac{64 M^2 L^2 (L+1) \log t}{\Delta^2_{min}}$, then for $C_{i,t,m_i(t)} = 2M \sqrt{\frac{(L+1) \log t}{m_i(t)}}$,
\begin{align*}
2 \sum_{i \in a_{t+1}} C_{i,t,m_i(t)} &= 2 \sum_{i \in a_{t+1}} 2M \sqrt{\frac{(L+1) \log t}{m_i(t)}} \\
	& < 4 \sum_{i \in a_{t+1}} M \sqrt{\frac{(L+1) \log t}{\frac{64 M^2 L^2 (L+1) \log t}{\Delta^2_{min}}}} \\
	& = 4 \sum_{i \in a_{t+1}} M \Delta_{min} \frac{1}{8ML} \\
	& \leq \frac{\Delta_{min}}{2}.
\end{align*}
The second term requires $\sigma_u^{(a_t)} - \sigma_l^{(a_t)} > \frac{\alpha \Delta_{min}}{2 \varphi (\Phi^{-1}(\alpha))}$.

Now, $\sigma_u^{(a_{t+1})} = \sqrt{\sum_{i \in a_{t+1}} \min \left\{ \left( s_{i, m_i(t)}^2 + 2M^2 \sqrt{\frac{2(L+1)\log t}{m_i(t) - 1}} \right) , M^2 \right\}}$ and \\ $\sigma_l^{(a_{t+1})} = \sqrt{\sum_{i \in a_{t+1}} \max \left\{ \left( s_{i, m_i(t)}^2 - M^2 \sqrt{\frac{2(L+1)\log t}{m_i(t) - 1} + \frac{4(L+1)^2 (\log t)^2}{(m_i(t) - 1)^2}} \right), N^2 \right\} }$.

For simplicity these can be rewritten as $\sigma_u^{(a_{t+1})} = || ( \sigma_{u,i}^{(a_{t+1})} )_i||_2$ and $\sigma_l^{(a_{t+1})} = || (\sigma_{l,i}^{(a_{t+1})} )_i ||_2$, where
\begin{align*}
\sigma_{u,i}^{(a_{t+1})} & = \sqrt{ \min \left\{ \left( s_{i, m_i(t)}^2 + 2M^2 \sqrt{\frac{2(L+1)\log t}{m_i(t) - 1}} \right) , M^2 \right\} }\text{ and} \\
\sigma_{l,i}^{(a_{t+1})} & = \sqrt{ \max \left\{ \left( s_{i, m_i(t)}^2 - M^2 \sqrt{\frac{2(L+1)\log t}{m_i(t) - 1} + \frac{4(L+1)^2 (\log t)^2}{(m_i(t) - 1)^2}} \right), N^2 \right\} }.
\end{align*}
So,
\begin{align*}
\sigma_u^{(a_t)} - \sigma_l^{(a_t)} & = || (\sigma_{u,i}^{(a_t)})_i ||_2 - || (\sigma_{l,i}^{(a_t)})_i ||_2 \\
	& \leq || (\sigma_{u,i}^{(a_t)})_i - (\sigma_{l,i}^{(a_t)})_i ||_2 \\
	& \leq \sqrt{\sum_{i \in a_{t+1}} \left( \sigma_{u,i}^{(a_{t+1})} - \sigma_{l,i}^{(a_{t+1})} \right)^2}.
\end{align*}
Now, we have
\begin{align*}
\sqrt{b} - \sqrt{a} = \int_a^b \frac{1}{2\sqrt{x}} dx \leq \frac{1}{2\sqrt{a}} \int_a^b dx = \frac{b-a}{2 \sqrt{a}},
\end{align*}
so
\begin{align*}
\sigma_{u,i}^{(a_{t+1})} - \sigma_{l,i}^{(a_{t+1})} & \leq \frac{1}{2\sqrt{N^2}} \bigg( \min \left\{ \left( s_{i, m_i(t)}^2 + 2M^2 \sqrt{\frac{2(L+1)\log t}{m_i(t) - 1}} \right) , M^2 \right\} \\
	& - \max \left\{ \left( s_{i, m_i(t)}^2 - M^2 \sqrt{\frac{2(L+1)\log t}{m_i(t) - 1} + \frac{4(L+1)^2 (\log t)^2}{(m_i(t) - 1)^2}} \right), N^2 \right\} \bigg) \\
	& \leq \frac{1}{2\sqrt{N^2}} \bigg[ \bigg( s_{i, m_i(t)}^2 + 2M^2 \sqrt{\frac{2(L+1)\log t}{m_i(t) - 1}} \bigg) \\
	& - \bigg( s_{i, m_i(t)}^2 - M^2 \sqrt{\frac{2(L+1)\log t}{m_i(t) - 1} + \frac{4(L+1)^2 (\log t)^2}{(m_i(t) - 1)^2}} \bigg) \bigg] \\
	& \leq \frac{3M^2}{2N} \sqrt{\frac{2(L+1)\log t}{m_i(t) - 1} + \frac{4(L+1)^2 (\log t)^2}{(m_i(t) - 1)^2}}.
\end{align*}
So,
\begin{align*}
\sigma_u^{(a_t)} - \sigma_l^{(a_t)} & \leq \frac{3M^2}{2N} \sqrt{\sum_{i \in a_{t+1}} \frac{2(L+1)\log t}{m_i(t) - 1} + \frac{4(L+1)^2 (\log t)^2}{(m_i(t) - 1)^2}}
\end{align*}

For the above quantity to be less than $\frac{\alpha \Delta_{min}}{2 \varphi (\Phi^{-1}(\alpha))}$ it is sufficient that 
\begin{align*}
\sum_{i \in a_{t+1}} \frac{2(L+1)\log t}{m_i(t) - 1} + \frac{4(L+1)^2 (\log t)^2}{(m_i(t) - 1)^2} & < \frac{\alpha^2 \Delta^2_{min}}{4 \varphi (\Phi^{-1}(\alpha))^2} \frac{4N^2}{9M^4} = \frac{\alpha^2 \Delta^2_{min} N^2}{9M^4 \varphi (\Phi^{-1}(\alpha))^2} \text{ or}\\
\left( \frac{2(L+1)\log t}{m_i(t) - 1} \right) \left( \frac{2(L+1)\log t}{m_i(t) - 1} + 1 \right) & < \frac{\alpha^2 \Delta^2_{min} N^2}{9 L M^4 \varphi (\Phi^{-1}(\alpha))^2}, \text{ or} \\
\frac{2(L+1)\log t}{m_i(t) - 1} + 1 & < \frac{\alpha \Delta_{min} N}{3 \sqrt{L} M^2 \varphi (\Phi^{-1}(\alpha))}, \text{ or} \\
m_i(t) - 1 & > \frac{12 \sqrt{L} (L+1) M^2 \varphi (\Phi^{-1}(\alpha)) \log t}{\alpha \Delta_{min} N}.
\end{align*}
Therefore, letting $l = \max \left\{ \frac{64 M^2 L^2 (L+1) \log T}{\Delta^2_{min}}, \frac{12 \sqrt{L} (L+1) M^2 \varphi (\Phi^{-1}(\alpha)) \log T}{\alpha \Delta_{min} N} + 1 \right\}$ gives
\begin{align*}
\mathbb{E} [T_i(T)] & \leq 2 + \max \left\{ \frac{64 M^2 L^2 (L+1) \log T}{\Delta^2_{min}}, \frac{12 \sqrt{L} (L+1) M^2 \varphi (\Phi^{-1}(\alpha)) \log T}{\alpha \Delta_{min} N} + 1 \right\} + \sum_{t=l+1}^T 4 t^{-2}L \\
& \leq 2 + \max \left\{ \frac{64 M^2 L^2 (L+1) \log T}{\Delta^2_{min}}, \frac{12 \sqrt{L} (L+1) M^2 \varphi (\Phi^{-1}(\alpha)) \log T}{\alpha \Delta_{min} N} + 1 \right\} + \frac{2}{3} \pi^2 L
\end{align*}
Therefore, the regret is bounded by
\begin{align*}
\mathcal{R}_{\mathsf{CVaR}_\alpha} (T) & \leq \sum_{i=1}^K \Delta_{max} \mathbb{E} [T_i(T)] \\
	& \lesssim \frac{4M^2 \sqrt{L}(L+1) K \log T \Delta_{max}}{\Delta_{min}} \max\left\{ \frac{16 L\sqrt{L}}{\Delta_{min}}, \frac{3 \varphi (\Phi^{-1}(\alpha))}{\alpha N} \right\} + 2K \Delta_{max} + \frac{2}{3} \pi^2 LK \Delta_{max}.
\end{align*}

\section{Proof of Theorem \ref{thm:regret_bounded}}

We rewrite the regret of the algorithm $\pi$ as
\begin{align*}
\mathcal{R}_{\mathsf{CVaR}_\alpha}(T) & = \mathbb{E} \left[ \sum_{t=1}^T \mathbbm{1} \{ \Delta_{a_t} \neq 0 \} \Delta_{a_t} \right] \\
	& = \mathbb{E} \left[ \sum_{t=1}^T \mathbbm{1} \{ \mathcal{E}_t \} \Delta_{a_t} \right] + \mathbb{E} \left[ \sum_{t=1}^T \mathbbm{1} \{ \neg \mathcal{E}_t \} \Delta_{a_t} \right],
\end{align*}
where $a_t$ is the super arm chosen at time $t$, $\Delta_a = \mathsf{CVaR}_\alpha(a^*) - \mathsf{CVaR}_\alpha(a)$ is the $\mathsf{CVaR}$ gap, and
$$\mathcal{E}_t = \left\{ \exists \; j \in [K] \text{ s.t } \sup_{x \in [0,1]} | \hat{F}_{j, m_j(t-1)}(x) - F_j(x) | \geq C_{j, t-1, T_{j, t-1}} \right\}$$
is the ``bad" event that the empirical distribution and true distribution of some arm are not close at some point in their domain.

For the first term of the regret, we know from the DKW inequality that, for $C_{i,t-1,T_{i,t-1}} = \sqrt{\frac{3\log (t)}{2T_{i,t-1}}}$,
\begin{align*}
\mathbb{P} \left( \mathcal{E}_t \right) & \leq \sum_{j=1}^K \sum_{s=1}^{t-1} \mathbb{P} \left( \sup_{x \in [0,1]} | \hat{F}_{j,s}(x) - F_j(x) | \geq C_{j, t-1, s} \right) \\
	& \leq \sum_{j=1}^K \sum_{s=1}^{t-1} 2 \exp \left(-2s C_{j,t-1,s}^2 \right) \\
	& = \sum_{j=1}^K \sum_{s=1}^{t-1} 2 \exp \left(-2s \frac{3\log (t)}{2s} \right) \\
	& = \sum_{j=1}^K \sum_{s=1}^{t-1} \frac{2}{t^3} \\
	& \leq \frac{2K}{t^2}.
\end{align*}
For the second term, since $\mathcal{E}_t$ does not occur, we have, for every $i \in [K]$,
\begin{align*}
\tilde{F}_i(x) < F_i(x) < \tilde{F}(x) + C_{i, t-1, T_{i,t-1}}.
\end{align*}
So, for every $a \in \mathcal{A}$,
$$
\tilde{F}_a(x) < F_a(x) < \tilde{F}_a(x) + 2 \sum_{j \in a} C_{j, t-1, T_{j, t-1}},
$$
and hence, using Proposition 4 of \cite{Cassel:AGeneralApproachToMABsUnderRiskCriteria}, we have
$$
\mathsf{CVaR}(F_a) < \mathsf{CVaR}(\tilde{F}_a) < \mathsf{CVaR}(F_a) + L w \left( 2 \sum_{i \in a} C_{i,t-1,T_{i,t-1}} \right),
$$
where
$$
w(x) = b(x + x^2), b = \frac{4}{\alpha \min\{ \alpha, 1 - \alpha \}}.
$$
Now, since $a_t$ has been chosen suboptimally, we have
\begin{align*}
0 < \Delta_{a_t} & = \mathsf{CVaR}(F_{a^*}) \mathsf{CVaR}(F_{a_t}) \\
	& < \mathsf{CVaR}(\tilde{F}_{a^*}) - \mathsf{CVaR}(\tilde{F}_{a_t}) + Lw \left( 2 \sum_{i \in a_t} C_{i,t-1,T_{i,t-1}} \right) \\
	& \leq Lb \left( 2 \sum_{j \in a_t} C_{j,t-1,T_{j, t-1}} + \left( 2 \sum_{j \in a_t} C_{j,t-1,T_{j, t-1}} \right)^2 \right) \\
	& \leq 2Lb \; \max \left\{ 2 \sum_{j \in a_t} C_{j,t-1,T_{j, t-1}}, \left( 2 \sum_{j \in a_t} C_{j,t-1,T_{j,t-1}} \right)^2 \right\} \\
	& = \phi^{-1}_L \left( 2 \sum_{j \in a_t} C_{j,t-1,T_{j, t-1}} \right),
\end{align*}
where $\phi^{-1}_L(x) = \max \{ 2Lbx, 2Lbx^2 \}$, and $\phi_L(y) = \min \{ \frac{y}{2Lb}, \sqrt{\frac{y}{2Lb}}\}$. So, we have
\begin{align*}
0 < \phi_L(\Delta_{a_t}) < 2 \sum_{i \in a_t} C_{i,t-1,T_{i,t-1}} = 2 \sum_{i \in a_t} \sqrt{\frac{3 \log t}{2 T_{i,t-1}}}.
\end{align*}
So,
\begin{align*}
\mathcal{R}_{\mathsf{CVaR}_\alpha}(T) & = \mathbb{E} \left[ \sum_{t=1}^T \mathbbm{1} \{ \Delta_{a_t} \neq 0 \} \Delta_{a_t} \right] \\
	& \leq K \Delta_{max} + \mathbb{E} \left[ \sum_{t=K+1}^T \mathbbm{1} \{ \neg \mathcal{E}_t \} \Delta_{a_t} \right] + \mathbb{E} \left[ \sum_{t=K+1}^T \mathbbm{1} \{ \mathcal{E}_t \} \Delta_{a_t} \right] \\
	& \leq K \Delta_{max} + \mathbb{E} \left[ \sum_{t=K+1}^T \mathbbm{1} \{ \neg \mathcal{E}_t \} \Delta_{a_t} \right] + \Delta_{max} \sum_{t=K+1}^T \mathbb{P} \{ \mathcal{E}_t \} \\
	& \leq K \Delta_{max} + \mathbb{E} \left[ \sum_{t=K+1}^T \mathbbm{1} \{ \neg \mathcal{E}_t \} \Delta_{a_t} \right] + \Delta_{max} \sum_{t=K+1}^T \frac{2k}{t^2} \\
	& \leq \left( 1 + \frac{\pi^2}{3} \right) k \Delta_{max} + \mathbb{E} \left[ \sum_{t=K+1}^T \mathbbm{1} \left\{ 0 < \phi_L(\Delta_{a_t}) < 2 \sum_{i \in a_t} \sqrt{\frac{3 \log t}{2 T_{i,t-1}}} \right\} \Delta_{a_t} \right].
\end{align*}
For simplicity, define the event
$$
\mathcal{H}_t = \left\{ 0 < \phi_L(\Delta_{a_t}) < 2 \sum_{i \in a_t} \sqrt{\frac{3 \log t}{2 T_{i,t-1}}} \right\}.
$$
Define two decreasing sequences of constants: $1 = \beta_0 > \beta_1 > \dots$ and $\alpha_1 > \alpha_2 > \dots$, such that $\lim_{k \rightarrow \infty} \alpha_k = 0, \lim_{k \rightarrow \infty} \beta_k = 0$, and satisfying
$$
\sqrt{6} \sum_{k=1}^\infty \frac{\beta_{k-1} - \beta_k}{\sqrt{\alpha_k}} \leq 1, \qquad \sum_{k=1}^\infty \frac{\alpha_k}{\beta_k} < 267.
$$
For $t \in \{ K+1, \dots, T \}$, let
$$
m_{k,t} = \begin{cases} \alpha_k \frac{L^2 \log T}{\left( \phi_L(\Delta_{a_t}) \right)^2} & \Delta_{a_t} > 0 \\ + \infty & \Delta_{a_t} = 0 \end{cases},
$$
and
$$
A_{k,t} = \{ i \in a_t | T_{i, t-1} \leq m_{k,t} \}.
$$
Define the event
$$
\mathcal{G}_{k,t} = \{ |A_{k,t}| \geq \beta_k L \}.
$$
\begin{lemma}
In the $t'$th round, if $\mathcal{H}_t$ happens, then there exists $k \in \mathbb{N}$ such that event $\mathcal{G}_{k,t}$ happens.
\end{lemma}
\begin{proof}
We need to prove $\mathcal{H}_t \implies \bigcup_{k=1}^\infty \mathcal{G}_{k,t}$. So assume $\mathcal{H}_t$ happens and none of $\mathcal{G}_{k,t}$ happen.
Let $A_{0,t} = a_t$ and $\bar{A}_{k,t} = a_t \setminus A_{k,t}$. Then, since $\lim_{k \rightarrow \infty} m_{k,t} = 0$, following Lemma 5 in \cite{Chen:CombMABGeneralRewardFuncs},
$$
\sum_{i \in a_t} \frac{1}{\sqrt{T_{i,t-1}}} < \sum_{k=1}^\infty \frac{(\beta_{k-1} - \beta_k)L}{\sqrt{m_{k,t}}}.
$$
Now, since $\mathcal{H}_t$ is assumed to happen, we have
\begin{align*}
\phi_L(\Delta_{a_t}) & < 2 \sum_{i \in a_t} \sqrt{\frac{3 \log t}{2 T_{i,t-1}}} \leq \sqrt{6 \log T} \sum_{i \in a_t} \frac{1}{\sqrt{T_{i,t-1}}} \\
	& < \sqrt{6 \log T} \sum_{k=1}^\infty \frac{(\beta_{k-1} - \beta_k)L}{\sqrt{m_{k,t}}} = \sqrt{6} \sum_{k=1}^\infty \frac{\beta_{k-1} - \beta_k}{\sqrt{\alpha_k}} \phi_L(\Delta_{a_t}) < \phi_L(\Delta_{a_t}),
\end{align*}
which is a contradiction, hence the Lemma.
\end{proof}
$\mathcal{G}_{k,t}$ is the event that at least $\beta_k K$ arms in the selected super arm do not have ``enough" ($>m_{k,t}$) number of samples. Let $\mathcal{G}_{i,k,t}$ be the corresponding event for a specific arm $i \in a_t$. That is,
$$
\mathcal{G}_{i,k,t} = \mathcal{G}_{k,t}  \wedge \left\{ i \in a_t, T_{i, t-1} \leq m_{k,t} \right\}.
$$
So, when $\mathcal{G}_{k,t}$ occurs, the corresponding event $\mathcal{G}_{i,k,t}$ occurs for atleast $\beta_k K$ arms, i.e,
$$
\mathbbm{1} \{ \mathcal{G}_{k,t}, \Delta_{a_t} > 0 \} \leq \frac{1}{\beta_k L} \sum_{i \in a_B} \mathbbm{1} \left\{ \mathcal{G}_{i,k,t}, \Delta_{a_t} > 0 \right\},
$$
where $a_B$ is the set of all arms that are part of at least one suboptimal super arm. Therefore,
\begin{align*}
\sum_{t=K+1}^T \mathbbm{1} \{ \mathcal{H}_t \} \Delta_{a_t} \leq \sum_{i \in a_B} \sum_{k=1}^\infty \sum_{t=K+1}^T \mathbbm{1} \left\{ \mathcal{G}_{i,k,t}, \Delta_{a_t} > 0 \right\} \frac{\Delta_{a_t}}{\beta_k L}
\end{align*}
For each arm $i \in a_B$, let the arm be contained in $N_i$ suboptimal super arms $a^B_{i,1}, a^B_{i,2}, \dots, a^B_{i,N_i}$. Let $\Delta_{i,l} = \Delta_{a^B_{i, l}}$ for $l \in [N_i]$, and without any loss in generality
$$+\infty = \Delta_{i,0} \geq \Delta_{i,1} \geq \Delta_{i,2} \geq \dots \geq \Delta_{i, N_i} = \Delta_{i,min}.$$
Then, we have
\begin{align*}
& \sum_{t=K+1}^T \mathbbm{1} \{ \mathcal{H}_t \} \Delta_{a_t} \\
& \leq \sum_{i \in a_B} \sum_{k=1}^\infty \sum_{t=K+1}^T \mathbbm{1} \left\{ \mathcal{G}_{i,k,t}, \Delta_{a_t} > 0 \right\} \frac{\Delta_{a_t}}{\beta_k L} \\
& \leq \sum_{i \in a_B} \sum_{k=1}^\infty \sum_{t=K+1}^T \sum_{l=1}^{N_i} \mathbbm{1} \left\{ T_{i, t-1} \leq m_{k,t}, a_t = a^B_{i,l} \right\} \frac{\Delta_{i,l}}{\beta_k L} \\
& = \sum_{i \in a_B} \sum_{k=1}^\infty \sum_{t=K+1}^T \sum_{l=1}^{N_i} \mathbbm{1} \left\{ T_{i, t-1} \leq \alpha_k \frac{L^2 \log T}{\left( \phi_L(\Delta_{i,l}) \right)^2}, a_t = a^B_{i,l} \right\} \frac{\Delta_{i,l}}{\beta_k L} \\
& = \sum_{i \in a_B} \sum_{k=1}^\infty \sum_{t=K+1}^T \sum_{l=1}^{N_i} \sum_{j=1}^l \mathbbm{1} \left\{ \alpha_k \frac{L^2 \log T}{\left( \phi_L(\Delta_{i,j-1}) \right)^2} < T_{i, t-1} \leq \alpha_k \frac{L^2 \log T}{\left( \phi_L(\Delta_{i,j}) \right)^2}, a_t = a^B_{i,l} \right\} \frac{\Delta_{i,l}}{\beta_k L} \\
& \leq \sum_{i \in a_B} \sum_{k=1}^\infty \sum_{t=K+1}^T \sum_{l=1}^{N_i} \sum_{j=1}^l \mathbbm{1} \left\{ \alpha_k \frac{L^2 \log T}{\left( \phi_L(\Delta_{i,j-1}) \right)^2} < T_{i, t-1} \leq \alpha_k \frac{L^2 \log T}{\left( \phi_L(\Delta_{i,j}) \right)^2}, a_t = a^B_{i,l} \right\} \frac{\Delta_{i,j}}{\beta_k L} \\
& \leq \sum_{i \in a_B} \sum_{k=1}^\infty \sum_{t=K+1}^T \sum_{l=1}^{N_i} \sum_{j=1}^{N_i} \mathbbm{1} \left\{ \alpha_k \frac{L^2 \log T}{\left( \phi_L(\Delta_{i,j-1}) \right)^2} < T_{i, t-1} \leq \alpha_k \frac{L^2 \log T}{\left( \phi_L(\Delta_{i,j}) \right)^2}, a_t = a^B_{i,l} \right\} \frac{\Delta_{i,j}}{\beta_k L} \\
& \leq \sum_{i \in a_B} \sum_{k=1}^\infty \sum_{t=K+1}^T \sum_{j=1}^{N_i} \mathbbm{1} \left\{ \alpha_k \frac{L^2 \log T}{\left( \phi_L(\Delta_{i,j-1}) \right)^2} < T_{i, t-1} \leq \alpha_k \frac{L^2 \log T}{\left( \phi_L(\Delta_{i,j}) \right)^2}, i \in a_t \right\} \frac{\Delta_{i,j}}{\beta_k L} \\
& \leq \sum_{i \in a_B} \sum_{k=1}^\infty \sum_{j=1}^{N_i} \left( \alpha_k \frac{L^2 \log T}{\left( \phi_L(\Delta_{i,j}) \right)^2} - \alpha_k \frac{L^2 \log T}{\left( \phi_L(\Delta_{i,j-1}) \right)^2} \right) \frac{\Delta_{i,j}}{\beta_k L} \\
& = L \log T \left( \sum_{k=1}^\infty \frac{\alpha_k}{\beta_k} \right) \sum_{i \in a_B} \sum_{j=1}^{N_i} \left( \frac{1}{\left( \phi_L(\Delta_{i,j}) \right)^2} - \frac{1}{\left( \phi_L(\Delta_{i,j-1}) \right)^2} \right) \Delta_{i,j} \\
& \leq 267 L \log T . \sum_{i \in a_B} \sum_{j=1}^{N_i} \left( \frac{1}{\left( \phi_L(\Delta_{i,j}) \right)^2} - \frac{1}{\left( \phi_L(\Delta_{i,j-1}) \right)^2} \right) \Delta_{i,j}
\end{align*}
Finally, for each $i \in a_B$, we have
\begin{align*}
\sum_{j=1}^{N_i} \left( \frac{1}{\left( \phi_L(\Delta_{i,j}) \right)^2} - \frac{1}{\left( \phi_L(\Delta_{i,j-1}) \right)^2} \right) \Delta_{i,j} & = \frac{\Delta_{i, N_i}}{\left( \phi_L(\Delta_{i, N_i}) \right)^2} + \sum_{j=1}^{N_i - 1} \frac{1}{\left( \phi_L(\Delta_{i, j}) \right)^2} \left( \Delta_{i,j} - \Delta_{i, j+1} \right) \\
& \leq \frac{\Delta_{i, N_i}}{\left( \phi_L(\Delta_{i, N_i}) \right)^2} + \int_{\Delta_{i,N_i}}^{\Delta_{i,1}} \frac{1}{\left( \phi_L(x) \right)^2} dx
\end{align*}
Now, $\phi_L(x) = \min \{ \frac{x}{2Lb}, \sqrt{\frac{x}{2Lb}}\}$, so for $x \in [\Delta_{i,N_i}, \Delta_{i,1}]$, $x \leq L$, and so $\frac{x}{2Lb} \leq \frac{1}{2b} < 1$, therefore, $\phi_L(x) = \frac{x}{2Lb}$, which gives
\begin{align*}
\sum_{j=1}^{N_i} \left( \frac{1}{\left( \phi_L(\Delta_{i,j}) \right)^2} - \frac{1}{\left( \phi_L(\Delta_{i,j-1}) \right)^2} \right) \Delta_{i,j} & \leq 4 L^2 b^2 \frac{\Delta_{i,N_i}}{\left( \Delta_{i,N_i} \right)^2} + 4 L^2 b^2 \int_{\Delta_{i,N_i}}^{\Delta_{i,1}} \frac{1}{x^2} dx \\
& = 4 L^2 b^2 \left( \frac{2}{\Delta_{i,N_i}} - \frac{1}{\Delta_{i,1}} \right) \\
& < \frac{8L^2 b^2}{\Delta_{i,min}}
\end{align*}
Therefore,
\begin{align*}
\sum_{t=K+1}^T \mathbbm{1} \{ \mathcal{H}_t \} \Delta_{a_t} \leq 2136 L^3 b^2 \log T \sum_{i \in a_B} \frac{1}{\Delta_{i,min}},
\end{align*}
giving
\begin{align*}
\mathcal{R}_{\mathsf{CVaR}_\alpha}(T) \leq C \frac{L^3}{\alpha^4} \log T \sum_{i \in a_B} \frac{1}{\Delta_{i,min}} + \left( 1 + \frac{\pi^2}{3} \right) k \Delta_{max}.
\end{align*}

\section{Proof of Theorem \ref{thm:regret_bounded_approximate}}

\subsection{Discretization}

Let $f_i \;($and $F_i), i \in [L]$ be $L$ probability mass functions (and cumulative distribution functions) corresponding to random variables $X_i$. Let $f' \; ($and $F'_i)$ be the corresponding distributions obtained by discretizing the support to points spaced $\epsilon$ apart by rounding up, resulting in corresponding random variables $X_i'$. That is, the probability mass (or density) at each point $x$ in the support of $f_i$ is shifted to $\ceil{\frac{x}{\epsilon}} \epsilon$.

This may result in multiple points being shifted to the same multiple of $\epsilon$, but for the sake of clarity, without any loss in generality, we keep track of all the original points (and their probability masses) individually even after merging. Let $'$ denote a distribution that has been snapped onto the $\epsilon-$grid this way. Now, we wish to study the relation between the distributions of $\sum_{i \in [L]} X_i$ and $\sum_{i \in [L]} X_i'$. (It should be noted if each of $L$ random variables are supported on multiples of $\epsilon$, their sum is supported on multiples of $\epsilon$ as well.)

Let $x_1, x_2, \dots x_L$ be points in the support of $f_1, f_2, \dots, f_L$ respectively. Then $\sum_{i \in [L]} x_i$ is in the support of $\bigoplus_{i \in [L]} f_i$, with probability mass $\prod_{i \in [L]} f_i(x_i)$ due to these $L$ specific points. Other combinations of points may contribute to this probability mass as well, but we will keep track of each individual contribution for the sake of clarity. Let $x_1', x_2', \dots x_L'$ be the points obtained by $\epsilon$-rounding up the corresponding points. These points are in the support of $f_1', f_2', \dots, f_L'$ respectively, and $\sum_{i \in [L]} x_i'$ is in the support of $\sum_{i \in [L]} X_i'$. Now,
\begin{align*}
\sum_{i \in [L]} x_i' - \sum_{i \in [L]} x_i & = \sum_{i=1}^L \ceil*{\frac{x_i}{\epsilon}} \epsilon - \sum_{i=1}^L x_i \\
	& = \left[ \sum_{i=1}^L \ceil*{\frac{x_i}{\epsilon}} - \left( \sum_{i=1}^L \frac{x_i}{\epsilon} \right) \right] \epsilon \\
	& < \left[ \sum_{i=1}^L \ceil*{\frac{x_i}{\epsilon}} - \left( \ceil*{\sum_{i=1}^L \frac{x_i}{\epsilon}} - 1 \right) \right] \epsilon \\
	& \leq \left[ \sum_{i=1}^L \ceil*{\frac{x_i}{\epsilon}} - \left( \ceil*{\sum_{i=1}^{L-1} \frac{x_i}{\epsilon}} + \ceil*{\frac{x_L}{\epsilon}} - 1 - 1 \right) \right] \epsilon \\
	& \leq \left[ \sum_{i=1}^L \ceil*{\frac{x_i}{\epsilon}} - \left( \ceil*{\sum_{i=1}^{L-2} \frac{x_i}{\epsilon}} + \sum_{i=L-1}^L \left( \ceil*{\frac{x_i}{\epsilon}} - 1 \right) - 1 \right) \right] \epsilon \\
	& \leq \left[ \sum_{i=1}^L \ceil*{\frac{x_i}{\epsilon}} - \left( \ceil*{\sum_{i=1}^{L-3} \frac{x_i}{\epsilon}} + \sum_{i=L-2}^L \left( \ceil*{\frac{x_i}{\epsilon}} - 1 \right) - 1 \right) \right] \epsilon \\
	& : \\
	& : \\
	& \leq \left[ \sum_{i=1}^L \ceil*{\frac{x_i}{\epsilon}} - \left( \ceil*{\sum_{i=1}^{L-j} \frac{x_i}{\epsilon}} + \sum_{i=L-j+1}^L \left( \ceil*{\frac{x_i}{\epsilon}} - 1 \right) - 1 \right) \right] \epsilon \\
	& : \\
	& : \\
	& \leq \left[ \sum_{i=1}^L \ceil*{\frac{x_i}{\epsilon}} - \left( \sum_{i=1}^L \left( \ceil*{\frac{x_i}{\epsilon}} - 1 \right) - 1 \right) \right] \epsilon \\
	& = (L + 1) \epsilon.
\end{align*}
In other words, when the points of each of $L$ individual discrete distributions are snapped onto an $\epsilon-$grid by rounding up, then the resultant sum of the distributions is no further than distance $(L+1) \epsilon$ to the right of the original sum of the distributions.

Now, we need to consider how this discretization affects the $\mathsf{CVaR}$ of the distributions. We have the following definition of $\mathsf{CVaR}$ that holds for discrete distributions:

$$
\mathsf{CVaR}_\alpha(X) = \frac{1}{\alpha} \left[ \sum_{x \leq x_\alpha} x f_X(x) - \left( \sum_{x \leq x_\alpha} f_X(x) - \alpha \right) x_\alpha \right],
$$
where
$$
x_\alpha = \mathsf{VaR}_\alpha(X) = \inf \{ x | F_X(x) \geq \alpha \}.
$$

Now, we have to consider the $\mathsf{CVaR}$ of some distribution $X'$ where some points in the support of $X'$ have been moved to the right by a distance no more than $\epsilon$. (We can include the multiplicative factor $L+1$ inside $\epsilon$).

Since the probability mass shifted rightwards, $v'_\alpha \geq v_\alpha$, but since the shift was atmost $\epsilon$, we have $v'_\alpha < v_\alpha + \epsilon$.

Now, let us divide the points in the support of $X$ based on how they affect the $\mathsf{CVaR}$ of $X$ and $X'$ (before and after the shift). Points contribute to the $\mathsf{CVaR}$ iff they are not more than $v'_\alpha$. So we can divide each point into 4 sets, depending on whether they contributed to the $\mathsf{CVaR}$ of $X$ and/or the $\mathsf{CVaR}$ of $X'$. (When we refer to a point here, that point may contribute only a part of the probability mass ``at that point", since many combinations of sums can end up being the same value. This does not affect the analysis because the contribution from each ``partial" point add up in the definition of $\mathsf{CVaR}$).

~\\
$A$: Points that contribute to the $\mathsf{CVaR}$ of $X$

~\\
$B$: Points that contribute to the $\mathsf{CVaR}$ of $X'$

~\\
So, the set of all points is $(A \setminus B) \cup (B \setminus A) \cup (A \cap B) \cup (A^c \cap B^c)$.

Considering the contributions of $A \cap B$,

\begin{align*}
\sum_{x \in A \cap B} x' f_{X'} (x') - x'_\alpha \sum_{x \in A \cap B} f_{X'} (x') & \leq \sum_{x \in A \cap B} \left( (x + \epsilon) f_{X} (x) \right) - x_\alpha \sum_{x \in A \cap B} f_{X'} (x') \\
	& = \sum_{x \in A \cap B} x f_{X} (x) - x_\alpha \sum_{x \in A \cap B} f_{X} (x) + \epsilon \sum_{x \in A \cap B} f_X(x).
\end{align*}

The points in $B \setminus A$, those that contribute to $X'$ but not $X$, are those that are not greater than $x'_\alpha$, but are greater than $x_\alpha$. Their contribution is

\begin{align*}
\sum_{x > x_\alpha, x' \leq x'_\alpha} \left( x' f_{X'}(x') \right) - x'_\alpha \sum_{x > x_\alpha, x' \leq x'_\alpha} f_{X'} (x') & \leq x'_\alpha \sum_{x > x_\alpha, x' \leq x'_\alpha} \left( f_{X'}(x') \right) - x'_\alpha \sum_{x > x_\alpha, x' \leq x'_\alpha} f_{X'} (x') \\
	& = 0.
\end{align*}

\begin{align*}
\mathsf{CVaR}_\alpha(X') & = \frac{1}{\alpha} \left[ \sum_{x' \leq x'_\alpha} x' f_X'(x') - \left( \sum_{x' \leq x'_\alpha} f_X'(x') - \alpha \right) x'_\alpha \right] \\
	& = \frac{1}{\alpha} \left[ \sum_{x \in A \cap B} x' f_X'(x') + \sum_{x \in B \setminus A} x' f_X'(x') - \left( \sum_{x' \in A \cap B} f_X'(x') + \sum_{x' \in B \setminus A} f_X'(x') - \alpha \right) x'_\alpha \right] \\
	& \leq \frac{1}{\alpha} \left[ \sum_{x \in A \cap B} x f_X(x) - x_\alpha \sum_{x \in A \cup B} f_X'(x') - \alpha x'_\alpha + \epsilon \sum_{x \in A \cap B} f_X(x) \right] \\
	& \qquad \qquad \qquad \qquad \qquad \qquad \qquad \qquad \text{ (from the previous two equations)} \\
	& = \frac{1}{\alpha} \left[ \sum_{x \in A} x f_X(x) - x_\alpha \sum_{x \in A} f_X'(x') - \alpha x'_\alpha - \sum_{x \in A \setminus B} x f_X(x) + x_\alpha \sum_{x \in A \setminus B} f_X'(x') + \epsilon \sum_{x \in A \cap B} f_X(x) \right]
\end{align*}

Now, for $x \in A \setminus B$, $x + \epsilon \geq x' > x'_\alpha$, so $-x < - x'_\alpha + \epsilon$, which gives
\begin{align*}
\mathsf{CVaR}_\alpha(X') & \leq \frac{1}{\alpha} \bigg[ \sum_{x \in A} x f_X(x) - x_\alpha \sum_{x \in A} f_X'(x') - \alpha x'_\alpha - x'_\alpha \sum_{x \in A \setminus B} f_X(x) \\
& \qquad \qquad + x_\alpha \sum_{x \in A \setminus B} f_X'(x') + \epsilon \sum_{x \in A \cap B} f_X(x) + \epsilon \sum_{x \in A \setminus B} f_X(x) \bigg] \\
& \leq \frac{1}{\alpha} \left[ \sum_{x \in A} x f_X(x) - x_\alpha \sum_{x \in A} f_X(x) - \alpha x_\alpha \right] + \frac{\epsilon}{\alpha} \sum_{x \in A} f_X(x) \\
	& \leq \mathsf{CVaR}_\alpha(X) + \frac{\epsilon}{\alpha}.
\end{align*}
Therefore, if we discretize the CDF $\tilde{F}_i$ of each arm $i$ onto an $\epsilon-$grid, the $\mathsf{CVaR}$ of each super arm $a$ satisfies:
$$
\mathsf{CVaR}_\alpha(\tilde{F}_a) \leq CVaR_\alpha(F'_a) \leq \mathsf{CVaR}_\alpha(\tilde{F}_a) + \frac{\epsilon (L+1)}{\alpha}.
$$

\subsection{Regret}
We proceed similar to the proof of Theorem \ref{thm:regret_bounded}, rewriting the regret as
\begin{align*}
\mathcal{R}_{\mathsf{CVaR}_\alpha}(T) & = \mathbb{E} \left[ \sum_{t=1}^T \mathbbm{1} \{ \Delta_{a_t} \neq 0 \} \Delta_{a_t} \right] \\
	& = \mathbb{E} \left[ \sum_{t=1}^T \mathbbm{1} \{ \mathcal{E}_t \} \Delta_{a_t} \right] + \mathbb{E} \left[ \sum_{t=1}^T \mathbbm{1} \{ \neg \mathcal{E}_t \} \Delta_{a_t} \right],
\end{align*}
where $a_t$ is the arm chosen by the policy $\pi$ at time $t$, $\Delta_a = \mathsf{CVaR}_\alpha(a^*) - \mathsf{CVaR}_\alpha(a)$ is $\mathsf{CVaR}$ gap, and
$$\mathcal{E}_t = \left\{ \exists \; j \in [K] \text{ s.t } \sup_{x \in [0,1]} | \hat{F}_{j, m_j(t-1)}(x) - F_j(x) | \geq C_{j, t-1, T_{j, t-1}} \right\}$$
is the ``bad" event that the empirical distribution and true distribution of some arm are not close at some point on their domain.

For the first term of the regret, we know from the previous analysis that, for $C_{i,t-1,m_i(t-1)} = \sqrt{\frac{3\log (t)}{2T_{i,t-1}}}$,
\begin{align*}
\mathbb{P} \left( \mathcal{E}_t \right) \leq \frac{2K}{t^2}.
\end{align*}
For the second term in the regret, since $\mathcal{E}_t$ does not occur, we have
$$
\tilde{F}_a(x) < F_a(x) < \tilde{F}_a(x) + 2 \sum_{j \in a} C_{j, t-1, T_{j, t-1}},
$$
and hence
$$
\mathsf{CVaR}(F_a) < \mathsf{CVaR}(\tilde{F}_a) < \mathsf{CVaR}(F_a) + Lw\left( 2 \sum_{j \in a} C_{j, t-1, T_{j,t-1}} \right).
$$
Now, since a suboptimal super arm has been chosen based on the discretized distributions $F'_a$,
\begin{align*}
\Delta_{a_t} & = \mathsf{CVaR}(F_{a^*}) - \mathsf{CVaR}(F_{a_t}) \\
	& < \mathsf{CVaR}(\tilde{F}_{a^*}) - \mathsf{CVaR}(\tilde{F}_{a_t}) + Lw\left( 2 \sum_{j \in a} C_{j, t-1, T_{j,t-1}} \right) \\
	& \leq \mathsf{CVaR}(F'_{a^*}) - \mathsf{CVaR}(F'_{a_t}) + \frac{\epsilon(L+1)}{\alpha} + Lw\left( 2 \sum_{j \in a_t} C_{j, t-1, T_{j,t-1}} \right) \\
	& \leq \frac{\epsilon(L+1)}{\alpha} + \phi_L^{-1}\left( 2 \sum_{j \in a_t} C_{j, t-1, T_{j,t-1}} \right),
\end{align*}
which implies either $\frac{\Delta_{a_t}}{2} \leq \phi_L^{-1}\left( 2 \sum_{j \in a_t} C_{j, t-1, T_{j,t-1}} \right)$ or $\frac{\Delta_{a_t}}{2} \leq \frac{\epsilon(L+1)}{\alpha}$. When $\frac{\Delta_{a_t}}{2} \leq \phi_L^{-1}\left( 2 \sum_{j \in a_t} C_{j, t-1, T_{j,t-1}} \right)$, we have
\begin{align*}
0 < \phi_L\left(\frac{\Delta_{a_t}}{2}\right) < 2 \sum_{i \in a_t} C_{i,t-1,T_{i,t-1}} = 2 \sum_{i \in a_t} \sqrt{\frac{3 \log t}{2 T_{i,t-1}}}.
\end{align*}
So, we have
\begin{align*}
\mathcal{R}_{\mathsf{CVaR}_\alpha}(T) & = \mathbb{E} \left[ \sum_{t=1}^T \mathbbm{1} \{ \Delta_{a_t} \neq 0 \} \Delta_{a_t} \right] \\
	& \leq K \Delta_{max} + \mathbb{E} \left[ \sum_{t=K+1}^T \mathbbm{1} \{ \neg \mathcal{E}_t \} \Delta_{a_t} \right] + \mathbb{E} \left[ \sum_{t=K+1}^T \mathbbm{1} \{ \mathcal{E}_t \} \Delta_{a_t} \right] \\
	& \leq K \Delta_{max} + \mathbb{E} \left[ \sum_{t=K+1}^T \mathbbm{1} \{ \neg \mathcal{E}_t \} \Delta_{a_t} \right] + \Delta_{max} \sum_{t=K+1}^T \mathbb{P} \{ \mathcal{E}_t \} \\
	& \leq K \Delta_{max} + \mathbb{E} \left[ \sum_{t=K+1}^T \mathbbm{1} \{ \neg \mathcal{E}_t \} \Delta_{a_t} \right] + \Delta_{max} \sum_{t=K+1}^T \frac{2k}{t^2} \\
	& \leq \left( 1 + \frac{\pi^2}{3} \right) k \Delta_{max} + \mathbb{E} \left[ \sum_{t=K+1}^T \mathbbm{1} \left\{ 0 < \phi_L(\Delta_{a_t}) < 2 \sum_{i \in a_t} \sqrt{\frac{3 \log t}{2 T_{i,t-1}}} \right\} \Delta_{a_t} \right] \\
	& \qquad \qquad \qquad \qquad \quad + \mathbb{E} \left[ \sum_{t=K+1}^T \mathbbm{1} \left\{ \frac{\Delta_{a_t}}{2} \leq \frac{\epsilon (L+1)}{\alpha} \right\} \Delta_{a_t} \right]
\end{align*}

From the proof of Theorem \ref{thm:regret_bounded}, it can be seen that, for some constant $C$,
\begin{align*}
    \mathbb{E} \left[ \sum_{t=K+1}^T \mathbbm{1} \left\{ 0 < \phi_L(\Delta_{a_t}) < 2 \sum_{i \in a_t} \sqrt{\frac{3 \log t}{2 T_{i,t-1}}} \right\} \Delta_{a_t} \right] \leq C L^3 b^2 \log T \sum_{i \in a_B} \frac{1}{\Delta_{i,min}}.
\end{align*}
Therefore,
\begin{align*}
\mathcal{R}_{\mathsf{CVaR}}(\pi) & \leq \left( 1 + \frac{\pi^2}{3} \right) k \Delta_{max} + C L^3 b^2 \log T \sum_{i \in a_B} \frac{1}{\Delta_{i,min}} + 2 \epsilon \frac{(L+1)T}{\alpha}.
\end{align*}
So, letting $\epsilon = \frac{\alpha}{(L+1)T}$ gives the regret
$$
\mathcal{R}_{\mathsf{CVaR}_\alpha} (T) \leq 2 + \left( 1 + \frac{\pi^2}{3} \right) k\Delta_{max} + C \frac{L^3}{\alpha^4} \log T \sum_{i \in a_B} \frac{1}{\Delta_{i,min}}.
$$
\end{document}